\documentclass[12pt,reqno]{amsart}
\usepackage{graphicx}
\advance\hoffset by -0.5 truecm
\usepackage{amssymb}
\newtheorem{Theorem}{Theorem}[section]
\newtheorem{Lemma}[Theorem]{Lemma}

\begin{document}
\title[Constructive universal approximation]{Geometric separation and constructive universal approximation with two hidden layers}

\author[C. Sung]{Chanyoung Sung}
 \address{Dept. of Mathematics Education  \\
          Korea National University of Education \\
          Cheongju, Korea
          }
\email{cysung@kias.re.kr}



\date{}


\begin{abstract}
We give a geometric construction of neural networks that separate disjoint compact subsets of $\Bbb R^n$, and use it to obtain a constructive universal approximation theorem. Specifically, we show that networks with two hidden layers and either a sigmoidal activation (i.e., strictly monotone bounded continuous) or the ReLU activation can approximate any real-valued continuous function on an arbitrary compact set $K\subset\Bbb R^n$ to any prescribed accuracy in the uniform norm. For finite $K$, the construction simplifies and yields a sharp depth-2 (single hidden layer) approximation result.
\end{abstract}

\maketitle


\section{Introduction}

The well-known universal approximation theorem \cite{cybenko,hornik,leshno,pinkus} asserts that a feedforward neural network with a single hidden layer and a suitable non-polynomial continuous activation function can approximate any real-valued continuous function on a compact set $K\subset \Bbb R^n$ arbitrarily well (in the uniform norm). The standard proofs are largely non-constructive, relying on existence results from functional analysis (often via separation arguments such as the Hahn–Banach theorem). As a consequence, these proofs do not directly yield an explicit approximation procedure, nor do they provide quantitative bounds on the required width (i.e., the number of hidden units) of the approximating network. Although Monico \cite{monico} gave a very elementary argument showing that three hidden layers suffice for universal approximation under a sigmoidal activation, that proof likewise remains non-constructive in the sense that it does not furnish an explicit network construction with quantitative size control.

While the classical universal approximation theorem is often proved by functional-analytic arguments that are non-constructive, a number of constructive approaches have been developed in various settings (e.g., \cite{barron, ito, li, hidalgo, sonoda, yarotsky}). However, “constructive” results in this area typically come with different trade-offs, and no single method uniformly dominates along all axes of interest.

First, many constructions are tailored to particular classes of activation functions. For instance, Barron’s original scheme is closely tied to sigmoidal-type superpositions and to the analytic structure exploited in that framework, and related extensions likewise rely on activation-specific properties. Other shallow-network constructions require additional regularity assumptions on the activation (such as bounded variation or smoothness) in order to implement an explicit approximation mechanism. Second, some approaches impose restrictions either on the geometry/topology of the underlying input set $K$ (e.g., requiring a “tame” representation such as triangulability or a convenient geometric decomposition) or on the regularity of the target function (e.g., smoothness assumptions needed to obtain rates). Third, certain constructive proofs achieve universality only by allowing very large depth, or else by producing extremely wide shallow networks; while mathematically valid, such architectures can be difficult to interpret or to leverage algorithmically.

In this paper we contribute a complementary point in this landscape. We prove a constructive universal approximation theorem for networks with two hidden layers (depth 3 in our notation): for any compact set $K\subset \Bbb R^n$ and any $f\in C(K)$, we explicitly construct a depth-3 network that approximates $f$ arbitrarily well. Rather than reproving the sharp depth-2 (single hidden layer) universality for general non-polynomial activations, we relax the depth by one layer and use the resulting geometric flexibility to implement a Urysohn-style separation principle. Concretely, our second hidden layer acts as a “selector/aggregator” that combines finitely many separation gadgets, making the choice of first-layer parameters more transparent than forcing all separation and aggregation into a single hidden layer. When $K$ is finite, our construction reduces to the sharp depth-2 approximation statement.

Our method is constructive in a geometric/topological sense: the approximating network is produced by a Tietze-style iterative reduction of oscillation, where each step uses Urysohn-type separation lemmas realized by an explicit finite sequence of geometric operations (coverings, separation gadgets, and a symmetrization/averaging procedure) in the ambient $\Bbb R^n$. This framework treats, within a single proof template, both strictly monotone bounded (“sigmoidal”) activations and the ReLU activation—respectively classical and modern standard choices—and it is flexible enough that similar constructions may extend to other activations after suitable modifications. The main drawback is that the resulting width can be very large, so the construction is not intended as an efficient practical algorithm; rather, it highlights a trade-off between small depth and explicit geometric realizability.

We hope that this viewpoint sheds additional insight on constructive approximation by neural networks, and that the separation lemmas developed along the way may be useful beyond the present application.

\section{Preliminaries}
Let $\sigma :\Bbb R\rightarrow \Bbb R$ be any bounded continuous function which is either strictly increasing or strictly decreasing. It serves as an activation function of neural networks and the typical example of it is the sigmoid function $\sigma_{sig}(x)=1/(1+e^{-x})$ or the hyperbolic tangent function $\tanh(x)=(e^x-e^{-x})/(e^x+e^{-x})$. We shall call such $\sigma$ a sigmoidal activation function.

Let $K$ be any compact subset of $\Bbb R^n$, and $C(K)$ be the space of real-valued continuous functions defined on $K$. We equip $C(K)$ with the supremum norm  $||\cdot||_\infty$ so that it becomes a Banach space. To approximate $C(K)$ we introduce the following subspaces as in \cite{monico} :
$$\mathcal{N}_1=\{f\in C(K)| f(x_1,\cdots,x_n)=a_0+a_1x_1+\cdots+a_nx_n\ \textrm{for}\ \textrm{some}\ a_0,\cdots,a_n\in \Bbb R\}$$
$$\mathcal{N}_1^\sigma=\{F\in C(K)| F=\sigma\circ f\ \textrm{for}\ \textrm{some}\ f\in \mathcal{N}_1\}$$
$$\mathcal{N}_{k+1}=\{g\in C(K)| g=a_0+a_1F_1+\cdots+a_mF_m\ \textrm{for}\ \textrm{some}\ F_i\in \mathcal{N}_k^\sigma,a_i\in \Bbb R\}$$
$$\mathcal{N}_{k+1}^\sigma=\{G\in C(K)| G=\sigma\circ g\ \textrm{for}\ \textrm{some}\ g\in \mathcal{N}_{k+1}\}$$
for $k\geq 1$.  These function spaces can be defined for any subset $K\subseteq \Bbb R^n$ and summarized briefly as :
\begin{itemize}
  \item  $\mathcal{N}_1$ : affine functions
  \item  $\mathcal{N}_2$ : neural networks with one hidden layer, i.e. linear combinations of activated affine functions)
  \item  $\mathcal{N}_3$ : neural networks with two hidden layers
\end{itemize}

A simple but important fact we shall often use is that $\mathcal{N}_{k}^\sigma\subset \mathcal{N}_{k+1}$, and if $g_1,g_2\in \mathcal{N}_{k}$, then $a_0+a_1g_1+a_2g_2\in \mathcal{N}_{k}$ for any $a_i\in \Bbb R$.
The assertion of the universal approximation theorem is that $\mathcal{N}_2$ is dense in $C(K)$, and C. Monico \cite{monico} proved that $\mathcal{N}_4$ is dense in $C(K)$.
We shall prove that $\mathcal{N}_3$ is dense in $C(K)$.

Note that for any such $\sigma$ there exist constants $b,c\in \Bbb R$ such that $b+c\sigma$ is a strictly increasing continuous function with image $(0,1)$, and hence $\mathcal{N}_{k}$ is actually equal to $\mathcal{N}_{k}$ obtained by using $b+c\sigma$ instead of $\sigma$. Thus for the proof of $\overline{\mathcal{N}_k}=C(K)$ one may assume that $\sigma$ is any strictly increasing continuous function with image $(0,1)$, as is common in many papers. In the sequel we shall assume it.

The second kind of activation functions in neural networks is the ReLU function $$\tau(x):=\max(0,x)$$ which is also widely used in deep learning. In the same way as above, one can define
$\mathcal{N}_k$ and $\mathcal{N}_k^\tau$ using $\tau$ as activation. We shall also prove that $\mathcal{N}_3$ with $\tau$ as activation is dense in $C(K)$.

Throughout the paper, we fix the following notations. For any $x=(x_1,\cdots,x_n)\in \Bbb R^n$ and any $\varepsilon>0$, define an $n$-dimensional open cube
$$C_\varepsilon(x):=(x_1-\varepsilon, x_1+\varepsilon)\times \cdots\times (x_n-\varepsilon, x_n+\varepsilon),$$
an $n$-dimensional open ball $$B_\varepsilon(x):=\{y\in \Bbb R^n | \ ||x-y||<\varepsilon\},$$ and
an $(n-1)$-dimensional round sphere $$S_\varepsilon(x):=\{y\in \Bbb R^n | \ ||x-y||=\varepsilon\}.$$  The origin $(0,\cdots,0)\in \Bbb R^n$ is denoted by $\bold{o}$.

\section{Urysohn-type separation lemmas}
In view of the Urysohn lemma \cite{munk} asserting that any two disjoint closed subsets in $\Bbb R^n$ can be separated by a continuous function, it is crucial for the proof of $\overline{\mathcal{N}_3}=C(K)$ to show that $\mathcal{N}_3$ has sufficiently many functions enough to separate any two disjoint compact subsets.
Let's start with a point separation lemma. We show that $\mathcal{N}_2$ can separate arbitrarily apart in case of a closed set and a point :
\begin{Lemma}\label{pt-closedset-1}
Let $A$ and $\{\bold{p}\}$ be nonempty disjoint closed subsets in $\Bbb R^n$. Then for any $\epsilon>0$ there exists $h\in \mathcal{N}_{2}$ such that
$h(\bold{p})<\epsilon$, $h>1-\epsilon$ on $A$, and $h(\Bbb R^n)\subseteq (0,1)$.
\end{Lemma}
\begin{proof}
Since $A$ and  $\{\bold{p}\}$ are disjoint and closed, $$d(\bold{p},A):=\inf\{||x-\bold{p}||\ |\ x\in A\}>0.$$
(If $d(\bold{p},A)$ were zero, there would exist a sequence $\bold{p}_1, \bold{p}_2,\cdots$ in $A$ such that $$||\bold{p}_i-\bold{p}||<\frac{1}{i}.$$ This means that the sequence converges to $\bold{p}$. Since $A$ is closed, the limit point $\bold{p}$ must be in $A$, which is a contradiction.)

Let $\delta\in (0,d(\bold{p},A)/2)$.
Without loss of generality (WLOG) we may assume that $\epsilon<1/2$. We shall construct such a separating function as a hole-like function around $\bold{p}\in \Bbb R^n$. To cook it up, let's start with a simple 1-D example. Fix an integer $N$ sufficiently large so that
\begin{eqnarray}\label{mywife}
\frac{1-\frac{\epsilon}{N}}{1+\frac{\epsilon}{2nN}}>1-\epsilon.
\end{eqnarray}
(As $N\rightarrow \infty$, the LHS which is smaller than 1 for any $N>0$ tends to 1.) Since $$\lim_{x\rightarrow -\infty}\sigma(x)=0 \  \  \  \  \textrm{and}\  \  \  \ \lim_{x\rightarrow \infty}\sigma(x)=1,$$ one can take a (large) constant $c>0$ such that $$\sigma(c(x-\frac{3\delta}{2}))<\frac{\epsilon}{2nN}\ \ \  \textrm{on}\ (-\infty,\delta] \ \ \  \textrm{and}\ \ \  \sigma(c(x-\frac{3\delta}{2}))>1-\frac{\epsilon}{3N}\ \ \  \textrm{on}\ [2\delta, \infty).$$ Define $\psi:\Bbb R\rightarrow (0,1+\frac{\epsilon}{2nN})$ by $$\psi(x)=\sigma(c(x-\frac{3\delta}{2}))+\sigma(c(-x-\frac{3\delta}{2})).$$

\begin{figure}[htbp]
  \centering
  \includegraphics[width=0.9\textwidth]{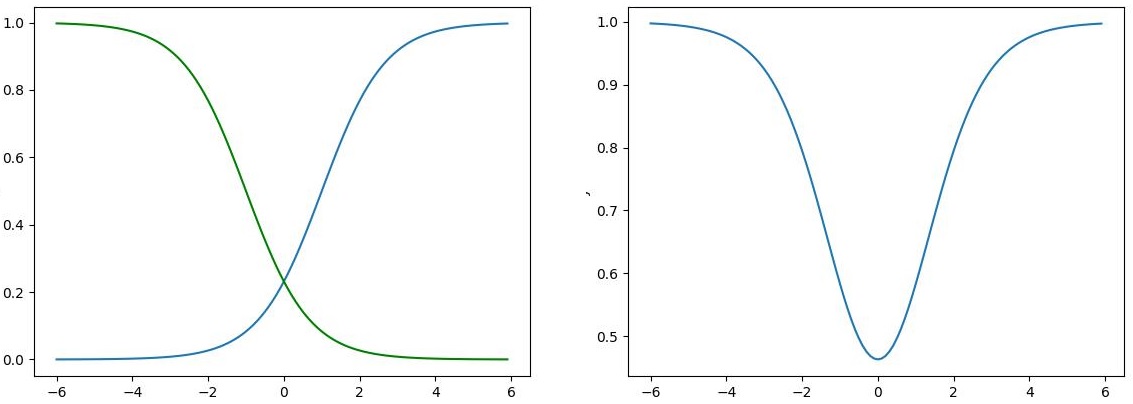}
  \caption{$\sigma(c(x-\frac{3\delta}{2})), \sigma(c(-x-\frac{3\delta}{2}))$ and $\psi$}
  \label{fig:myphoto-1}
\end{figure}

Here the upper bound of the even function $\psi$ is obtained from $$\psi(x)\leq \sigma(c(x-\frac{3\delta}{2}))+\sigma(c(-0-\frac{3\delta}{2}))< 1+\frac{\epsilon}{2nN}\ \ \textrm{for}\ x\geq 0,$$ since $\sigma(c(-x-\frac{3\delta}{2}))$ is strictly decreasing. (When $\sigma$ is $\sigma_{sig}$, one can directly compute that $\psi(x)$ is actually less than $1$ for any $x$, but the above bound is sufficient for our purpose.)


This $\psi$ is the desired hole-like function around $0$, but in higher dimension $n$ we need a modification. First we define an $n$-dimensional function $$\Psi:\Bbb R^n\rightarrow (0,n(1+\frac{\epsilon}{2nN}))$$ by $$\Psi(x):=\Sigma_{i=1}^n\psi(x_i)$$ which is less than $2n\cdot\frac{\epsilon}{2nN}=\frac{\epsilon}{N}$ on $C_\delta(\bold{o})$ and greater than $1-\frac{\epsilon}{3N}$ on $C_{2\delta}(\bold{o})^c$. Note that $\Psi$ is
strictly increasing w.r.t. the radial distance, i.e.
$$\Psi(\lambda x) > \Psi(x) > \Psi(\bold{o})$$
for any $\lambda>1$ and any $x\ne\bold{o}$, because $$\psi(\lambda x_i)\geq \psi(x_i)\geq \psi(0)$$ for all $i$ and there exists at least one nonzero $x_j$ so that $$\psi(\lambda x_j)> \psi(x_j)> \psi(0).$$
But $\Psi$ is far from being radially symmetric. So even if $x$ is far away from $C_{2\delta}(\bold{o})$, for $x$ in a thin set $$S:=\cup_{i=1}^n\pi_i^{-1}((-2\delta,2\delta))$$ where $\pi_i:\Bbb R^n\rightarrow \Bbb R$ is the $i$-th projection map, the lower bound $1-\frac{\epsilon}{3N}$ of $\Psi(x)$ is not big enough in compared to $\sup_x\Psi(x)\sim n$.
Below is drawn a thin set $S$ when $n$ is 2, and $S$ is the union of $n$ slabs in higher dimension $n$.
\begin{figure}[htbp]
  \centering
  \includegraphics[width=0.3\textwidth]{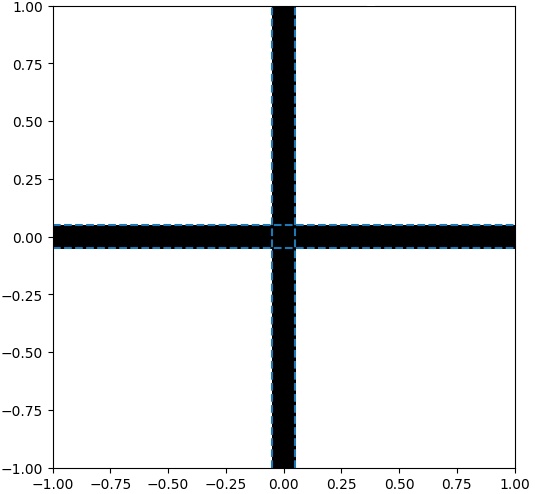}
  \caption{$\cup_{i=1}^n\pi_i^{-1}((-2\delta,2\delta))$ when $n=2$ and $2\delta=0.05$}
  \label{fig:myphoto-1}
\end{figure}

Outside the thin set, $\Psi$ is greater than $n(1-\frac{\epsilon}{3N})$. To exploit this nice property, we shall take an average of $\Psi$, which will make it almost radially symmetric and raise the values on the thin set $S$ sufficiently high. First consider the integral $$\bar{\Psi}(x):=\int_{SO(n)}\Psi(R(x))\ d\mu(R)$$ defined using a bi-invariant metric on a compact Lie group $SO(n)$ with the induced probability Haar measure $\mu$. The integration is over a compact set, so $\bar{\Psi}$ is a continuous function on $\Bbb R^n$ and radially symmetric by the left invariance of $\mu$. Moreover it is strictly increasing w.r.t. the radial distance, because
\begin{eqnarray*}
\bar{\Psi}(\lambda x) &=& \int_{SO(n)}\Psi(R(\lambda x))\ d\mu(R)\\ &=& \int_{SO(n)}\Psi(\lambda R(x))\ d\mu(R)\\ &>& \int_{SO(n)}\Psi(R(x))\ d\mu(R) \\
&=& \bar{\Psi}(x)\\ &>& \bar{\Psi}(\bold{o})=\Psi(\bold{o})
\end{eqnarray*}
for any $\lambda>1$ and any $x\ne\bold{o}$. The desired estimation of the lower bound of $\bar{\Psi}$ on $B_{d(\bold{p},A)}(\bold{o})^c$ comes from the small volume of $S$ and the transitivity of the $SO(n)$ action on each sphere $S_r(\bold{o})$. We will make it precise below using a little knowledge of Riemannian  geometry.

For $\bold{v}:=(d(\bold{p},A),0,\cdots,0)\in \Bbb R^n$, define
$$O_{\bold{v}}:=\{R\in SO(n)| R(\bold{v})\in S\}.$$
Note that the canonical isometric action of $SO(n)$ on $S_{||\bold{v}||}(\bold{o})$ is transitive, whose round metric is the normal homogeneous metric induced from a bi-invariant metric of $SO(n)$. With these metrics the quotient map $$\pi:SO(n)\rightarrow S_{||\bold{v}||}(\bold{o})$$ of the $SO(n)$ action becomes a Riemannian submersion with totally geodesic fibers isometric to $SO(n-1)$ with a bi-invariant metric. (\cite{Besse}) Since $\pi_*$ restricted to each horizontal space, i.e. the orthogonal complement to each tangent space of fiber, is isometric,  the volume $\mu(\pi^{-1}(U))$ for any open subset $U\subseteq S_{||\bold{v}||}(\bold{o})$ is equal to
$$\textrm{Vol}(U)\cdot \textrm{Vol}(SO(n-1))$$ where $\textrm{Vol}(\cdot)$ denotes volume computed using the corresponding metric.
Then there exists a constant $C'=C'(n,d(\bold{p},A))>0$ independent of $\delta$ such that
\begin{eqnarray*}
\mu(O_{\bold{v}}) &=& \mu(SO(n))\frac{\textrm{Vol}(S_{||\bold{v}||}(\bold{o})\cap S)}{\textrm{Vol}(S_{||\bold{v}||}(\bold{o}))}\\
&=& \mu(SO(n))\frac{n\int_{-2\delta}^{2\delta}(||v||^2-t^2)^{\frac{n-2}{2}}|S^{n-2}|\ dt}{\textrm{Vol}(S_{||\bold{v}||}(\bold{o}))}\\
&<& \mu(SO(n))\frac{n\int_{-2\delta}^{2\delta}(||v||^2)^{\frac{n-2}{2}}|S^{n-2}|\ dt}{\textrm{Vol}(S_{||\bold{v}||}(\bold{o}))}\\
&<& C'\delta
\end{eqnarray*}
where $|S^{n-2}|$ denotes the volume of the unit $S^{n-2}$. Therefore
\begin{eqnarray*}
\bar{\Psi}(\bold{v})&>& \int_{SO(n)-O_{\bold{v}}}\Psi(R(\bold{v}))\ d\mu(R)\\ &>& \int_{SO(n)-O_{\bold{v}}} n(1-\frac{\epsilon}{3N})\ d\mu(R) \\
&>& (1-C'\delta) n(1-\frac{\epsilon}{3N})\\ &>&  n(1-\frac{\epsilon}{2N})
\end{eqnarray*}
for any sufficiently small $\delta>0$. Taking such $\delta$, we have that for any $x\in B_{d(\bold{p},A)}(\bold{o})^c$
\begin{eqnarray}\label{unipole}
\bar{\Psi}(x)> n(1-\frac{\epsilon}{2N}).
\end{eqnarray}

However $\bar{\Psi}$ is not in $\mathcal{N}_{2}$ in general, so we should approximate it by a Riemann sum $\tilde{\Psi}$ which can serve as our hole-like function in $\mathcal{N}_{2}$.
Take a triangulation $\mathcal{P}_\varepsilon$ of $SO(n)$ which is a smooth manifold of dimension $\frac{n(n-1)}{2}$ such that the diameter of each $\frac{n(n-1)}{2}$-simplex $V_i$ of $\mathcal{P}_\varepsilon$ is less than $\varepsilon>0$ and define $\tilde{\Psi}:\Bbb R^n\rightarrow \Bbb R$ by $$\tilde{\Psi}(x):=\sum_{i=1}^{L}\Psi(R_i(x))\mu(V_i)$$ where $V_1,\cdots, V_L$ are all $\frac{n(n-1)}{2}$-simplices of $\mathcal{P}_\varepsilon$ and $R_i$ is any point in $V_i$.
Since each $R_i$ is a linear map, $\tilde{\Psi}$ belongs to $\mathcal{N}_{2}$ just as $\Psi$.
In the similar way to $\bar{\Psi}$, $\tilde{\Psi}$ is also strictly increasing w.r.t. the radial distance, and almost radially symmetric in the sense that
$$||\tilde{\Psi}-\bar{\Psi}||_{\infty,R}:=\sup \{|\tilde{\Psi}(x)-\bar{\Psi}(x)|\ |\ x\in \overline{B_R(\bold{o})}\}$$ for any $R>0$ can be made arbitrarily small by taking sufficiently fine $\mathcal{P}_\varepsilon$. (Considering $\Psi(R(x))$ as a function of $(R,x)$ in a compact space $SO(n)\times \overline{B_R(\bold{o})}$, it is uniformly continuous. That's why the Riemann sum $\tilde{\Psi}(x)$ over $SO(n)$ converges uniformly on $\overline{B_R(\bold{o})}$ to the integral $\bar{\Psi}(x)$, as $\varepsilon\rightarrow 0$.)

Certainly $\tilde{\Psi}$ takes values in $(0,n(1+\frac{\epsilon}{2nN}))$, since it is the average of $L$ values in $(0,n(1+\frac{\epsilon}{2nN}))$. Moreover $\tilde{\Psi}$ is less than $\epsilon/N$ on $B_\delta(\bold{o})$, since each $R_i(B_\delta(\bold{o}))$ is equal to $B_\delta(\bold{o})$.
What about $\tilde{\Psi}$ on $B_{d(\bold{p},A)}(\bold{o})^c$? We need it to satisfy
\begin{eqnarray}\label{dipole}
\tilde{\Psi}(x)> n(1-\frac{\epsilon}{N})\ \ \textrm{for}\ x\in B_{d(\bold{p},A)}(\bold{o})^c.
\end{eqnarray}

We claim that (\ref{dipole}) can be attained by taking sufficiently small $\delta>0$ first and then a sufficiently fine triangulation $\mathcal{P}_\varepsilon$ for $\varepsilon\ll \delta$.
By the radially increasing property of $\tilde{\Psi}$, it's enough to check it for $x\in  S_{d(\bold{p},A)}(\bold{o})$.
Indeed we have
\begin{eqnarray*}
\tilde{\Psi}(x)&\geq& \bar{\Psi}(x)-|\tilde{\Psi}(x)-\bar{\Psi}(x)|\\ &>&  n(1-\frac{\epsilon}{2N})-||\tilde{\Psi}-\bar{\Psi}||_{\infty,d(\bold{p},A)}\\ &>& n(1-\frac{\epsilon}{N})
\end{eqnarray*}
for $\delta$ sufficiently small enough to satisfy (\ref{unipole}) and also $\varepsilon\ll \delta$ sufficiently small enough to have $||\tilde{\Psi}-\bar{\Psi}||_{\infty,d(\bold{p},A)}<\frac{n\epsilon}{2N}$.

We now define $h:\Bbb R^n\rightarrow \Bbb (0,1)$ by $$h(x):=\frac{1}{n(1+\frac{\epsilon}{2nN})}\tilde{\Psi}(x-\bold{p})$$ which is greater than
$$\frac{n(1-\frac{\epsilon}{N})}{n(1+\frac{\epsilon}{2nN})}\ \ \ \ \textrm{on}\  B_{d(\bold{p},A)}(\bold{p})^c\supset A$$ and less than $$\frac{\epsilon}{N}\frac{1}{n(1+\frac{\epsilon}{2nN})}<\epsilon\ \ \ \ \textrm{on}\   B_\delta(\bold{p}).$$
Since $N$ was chosen so that (\ref{mywife}) holds, $h|_A> 1-\epsilon.$

\end{proof}

By superposing those hole-like functions constructed in the above lemma, one can separate two disjoint compact sets.
\begin{Lemma}\label{monicothm-1}
Let $A$ and $B$ be nonempty disjoint closed subsets in $\Bbb R^n$ so that $B$ is compact. Then for any $\epsilon>0$ there exist $b \in [\epsilon,1-\epsilon)$ and $H\in \mathcal{N}_{3}$ for $\sigma$ activation such that
$H<b$ on $B$, $H>1-\epsilon$ on $A$, and $H(\Bbb R^n)\subseteq (0,1)$.
\end{Lemma}
\begin{proof}
At each point $\bold{p}\in B$, there exists $h_{\bold{p}}\in \mathcal{N}_{2}$ such that
$$0<h_{\bold{p}}(\bold{p})<1/3 \ \ \ \ \textrm{and} \ \ \ \ 2/3<h_{\bold{p}}<1\ \ \textrm{on}\ A$$ and $h_{\bold{p}}(\Bbb R^n)\subseteq (0,1)$ by the above lemma. Choose an open ball $B(\bold{p})$ centered at $\bold{p}$ such that $h_{\bold{p}} < 1/3$ on $B(\bold{p})$. Thus we get an open cover $\cup_{\bold{p}\in B} B(\bold{p})$ of $B$ and hence a finite subcover $\cup_{i=1}^N B(\bold{p}_i)$ by the compactness of $B$.

WLOG we may assume that $\epsilon < 1/(N+1)$. Take $s, t \in \Bbb R$ such that
$$s+t\cdot \frac{1}{3} =\sigma^{-1}(\epsilon)\ \ \ \ \textrm{and} \ \ \ \ s+t\cdot \frac{2}{3}=\sigma^{-1}(1-\epsilon).$$
Define $H:\Bbb R^n\rightarrow (0,1)$ by $$H(x)=\frac{1}{N}\Sigma_{i=1}^N\sigma(s+th_{\bold{p}_i}(x)).$$ Then $H\in \mathcal{N}_{3}$ satisfies $$H>1-\epsilon\ \ \textrm{on}\ A\ \ \ \  \textrm{and} \ \  \ \ H<\frac{1}{N}(\epsilon+(N-1)\cdot 1)\ \ \textrm{on}\ B$$ where the 2nd inequality is obtained from the fact that any point in $B$ belongs to at least one of $B(\bold{p}_1),\cdots, B(\bold{p}_N)$. By simple computation using $\epsilon < 1/(N+1)$, $$1-\epsilon> \frac{\epsilon}{N}+1-\frac{1}{N}\geq \epsilon$$ and hence we can put $a:=1-\epsilon$ and $b:= \frac{\epsilon}{N}+1-\frac{1}{N}$.
\end{proof}


Similar lemmas can be obtained for ReLU activation too, and we need to construct a hole-like function in $\mathcal{N}_{2}$ with $\tau$ activation. As before we start with an 1-D example. For any positive constants $b,c$, and $b'>b$, let $\varphi_{b,c}:\Bbb R\rightarrow [0,\infty)$ and $\varphi_{b,c,b'}:\Bbb R\rightarrow [0,\infty)$ be defined by
$$\varphi_{b,c}(x):=\tau(c(x-b))+\tau(c(-x-b)), \ \ \ \ \ \ \varphi_{b,c,b'}(x):= \varphi_{b,c}(x)-\varphi_{b',c}(x).$$

\begin{figure}[htbp]
  \centering
  \includegraphics[width=\textwidth]{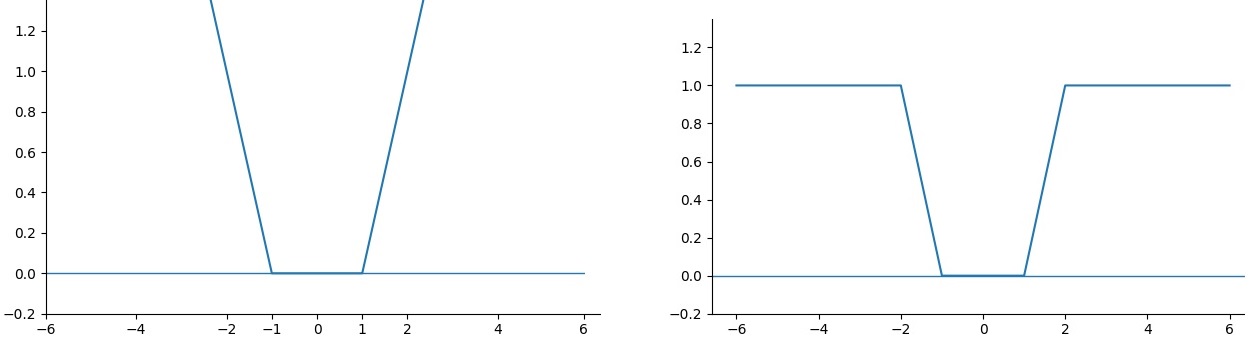}
  \caption{$\varphi_{1,1}$ and $\varphi_{1,1,2}$}
  \label{fig:myphoto-2}
\end{figure}
So in the graph of $\varphi_{b,c,b'}$, $c$ is the slope of the right-hand inclined line, $2b$ is the length of the bottom line, and $c(b'-b)$ is the height.
We use this $\varphi_{b,c,b'}$ to make an $n$-dimensional hole-like function and prove a point separation lemma for $\tau$ activation :
\begin{Lemma}\label{pt-closedset-2}
Let $A$ and $\{\bold{p}\}$ be nonempty disjoint closed subsets in $\Bbb R^n$. Then for any $\epsilon>0$ there exists $h\in \mathcal{N}_{2}$ with $\tau$ activation such that $h(\Bbb R^n)\subseteq [0,1]$, $h>1-\epsilon$ on $A$, and $h=0$ in an open neighborhood of $\bold{p}$.
\end{Lemma}
\begin{proof}
WLOG we may assume that $\epsilon <1$. The overall strategy of proof and notations follow those of the $\sigma$ case.
Let $\delta\in (0,d(\bold{p},A)/2)$.

When $n=1$, $\varphi_{\delta,\frac{1}{\delta}, 2\delta}(x-\bold{p})$ is a desired function $h$. In higher dimension $n$, we need to modify  $$\Phi(x)=\Sigma_{i=1}^n\varphi_{\delta,\frac{1}{\delta}, 2\delta}(x_i)$$ to make it almost radially symmetric.
As before, define $\tilde{\Phi}:\Bbb R^n\rightarrow [0,n]$ by a Riemann sum $$\tilde{\Phi}(x):=\frac{1}{L}\sum_{i=1}^{L}\Phi(R_i(x))\mu(V_i)$$ according to a triangulation $\mathcal{P}_\varepsilon$, which will be almost radially symmetric for sufficiently small $\varepsilon>0$. On this occasion $\Phi$ and $\tilde{\Phi}$ are 0 on $B_\delta(\bold{o})$, and non-strictly increasing w.r.t. the radial distance. By taking sufficiently small $\delta\gg \varepsilon >0$ we can arrange that $\tilde{\Phi}$ is greater than $n(1-\epsilon)$ on $B_{d(\bold{p},A)}(\bold{o})^c$.

Now define $h:\Bbb R^n\rightarrow \Bbb [0,1]$ by $$h(x):= \frac{1}{n}\tilde{\Phi}(x-\bold{p})$$ which
is greater than $1-\epsilon$ on $B_{d(\bold{p},A)}(\bold{p})^c\supset A$ and  equal to $0$ on $B_\delta(\bold{p})$. Therefore $h$ is a desired function in $\mathcal{N}_{2}$ separating $A$ and $\bold{p}$.
\end{proof}

\begin{Lemma}\label{monicothm-2}
Let $A$ and $B$ be nonempty disjoint closed subsets in $\Bbb R^n$ so that $B$ is compact. Then there exist $b\in [0,1)$ and $H\in \mathcal{N}_{3}$ for $\tau$ activation such that
$H\leq b$ on $B$, $H=1$ on $A$, and $H(\Bbb R^n)\subseteq [0,1]$.
\end{Lemma}
\begin{proof}
This is also proved in the similar way to Lemma \ref{monicothm-1}.
At each point $\bold{p}\in B$, there exist $h_{\bold{p}}\in \mathcal{N}_{2}$ and an open ball $B(\bold{p})$ such that
$$h_{\bold{p}}=0 \ \ \textrm{on}\ B(\bold{p}), \ \ \ \ h_{\bold{p}}>2/3 \ \ \textrm{on}\ A, \ \ \ \ h_{\bold{p}}(\Bbb R^n)\subseteq [0,1]$$ by the above lemma. From an open cover $\cup_{\bold{p}\in B} B(\bold{p})$ of $B$ we extract a finite subcover $\cup_{i=1}^N B(\bold{p}_i)$.

Define $H:\Bbb R^n\rightarrow [0,1]$ by $$H(x)=\frac{1}{N}\Sigma_{i=1}^N\varphi_{\frac{1}{3},3,\frac{2}{3}}(h_{\bold{p}_i}(x)).$$
Since $\varphi_{\frac{1}{3},3,\frac{2}{3}}\in \mathcal{N}_{2}$, $H$ belongs to $\mathcal{N}_{3}$.  From the properties
$$\varphi_{\frac{1}{3},3,\frac{2}{3}}=0\ \ \textrm{on}\ [0,1/3], \ \  \ \ \varphi_{\frac{1}{3},3,\frac{2}{3}}=1\ \ \textrm{on}\ [2/3,1], \ \ \ \ \varphi_{\frac{1}{3},3,\frac{2}{3}}(\Bbb R^n)=[0,1]$$ of $\varphi_{\frac{1}{3},3,\frac{2}{3}}$ and those of $h_{\bold{p}_i}$,
it readily follows that $$H=1\ \ \textrm{on}\ A, \ \ \ \ \ \ \ H\leq \frac{1}{N}(0+(N-1)\cdot 1)=1-\frac{1}{N}\ \ \textrm{on}\ B$$ where the inequality is obtained from the fact that any point in $B$ belongs to at least one of $B(\bold{p}_1),\cdots, B(\bold{p}_N)$. Now we can take the desired constant $b$ to be $1-1/N$.
\end{proof}

\section{Main Theorem}

We are now prepared to prove that $\mathcal{N}_3$ with sigmoidal or ReLU activation is dense in $C(K)$.
\begin{Theorem}
Let $K$ be a compact subset of $\Bbb R^n$ and $\mathcal{N}_{k}$ for each $k$ be the previously defined space with $\sigma$ or $\tau$ as activation.
For any $f\in C(K)$ and any $\epsilon>0$, there exists $\hat{f}\in \mathcal{N}_{3}$ such that $||f-\hat{f}||_\infty<\epsilon$.
\end{Theorem}
\begin{proof}
The idea of proof is inspired by the Tietze extension theorem \cite{munk}, and we can use our separation lemmas just as the Urysohn lemma is used in proving the Tietze's theorem. We prove both cases of activation together, so the following argument works for both $\sigma$ and $\tau$.
For any $g\in C(K)$, we denote $$M_g:=\max\{g(x)|x\in K\},\ \ \ m_g:=\min\{g(x)|x\in K\},\ \ \ w_g:=\frac{M_g-m_g}{3}.$$

WLOG we may assume that $f\notin \mathcal{N}_{3}$ and $m_f<0<M_f$ or equivalently
\begin{eqnarray}\label{COHwang}
-3w_f<m_f<0.
\end{eqnarray}
Since $K$ is compact and $f$ is continuous, $$E^+:=\{x\in K|\ M_f-w_f\leq f(x)\leq M_f\}\ \ \textrm{and} \ \  E^-:=\{x\in K|\ m_f\leq f(x)\leq m_f+w_f\}$$ are compact subsets of $\Bbb R^n$. They are nonempty, since $f$ is nonconstant. Applying Lemma \ref{monicothm-1} to $E^{\pm}$,
there exist $a,b \in (0,1)$ with $a>b$ and $H\in \mathcal{N}_{3}$ for $\sigma$ activation such that
$$H<b\ \  \textrm{on}\ E^-,\ \ \ \ \ H>a\ \ \textrm{on}\ E^+,\ \ \ \ \ H(\Bbb R^n)\subseteq [0,1].$$ In the $\tau$-case, we apply Lemma \ref{monicothm-2} and get $H\in \mathcal{N}_{3}$ (for $\tau$ activation) with the same above property. Let's set $c:=1-\frac{a-b}{3}\in (0,1)$.

Define $$g_0:=(m_f+w_f)+w_f H\in \mathcal{N}_{3}.$$ Then by (\ref{COHwang}), $-2w_f<g_0<2w_f $ on $K$, so $$||g_0||_\infty \leq \frac{2}{3}(M_f-m_f).$$
Let's estimate $f-g_0$ :
 $$0\leq f-g_0<w_f+w_f(1-a)\ \ \textrm{on}\ E^+, \ \ \ \ \ \  -(w_f+w_fb)<f-g_0\leq 0\ \ \textrm{on}\ E^-$$
$$|f-g_0|<w_f\ \ \textrm{on}\ K-(E^+\cup E^-).$$
Thus $$-w_f(1+b)<f-g_0<w_f(1+1-a)\ \ \ \textrm{on}\ K$$ and hence
$$M_{f-g_0}-m_{f-g_0}\leq w_f(3-a+b)=c\cdot(M_f-m_f).$$

Note that $f-g_0$ is nonconstant, because otherwise $f=g_0+\textrm{constant}\in \mathcal{N}_{3}$. We claim that $$m_{f-g_0}<0<M_{f-g_0}.$$
From the above estimate of $f-g_0$ on $E^-\ne \emptyset$, $m_{f-g_0}<0$ is immediately obtained.  Likewise from the above estimate of $f-g_0$ on $E^+\ne \emptyset$, $M_{f-g_0}>0$ follows.

Hence we can apply the above process to $f-g_0$ instead of $f$, we can obtain $g_1\in \mathcal{N}_3$ such that
$$||g_1||_\infty \leq \frac{2}{3}(M_{f-g_0}-m_{f-g_0})\leq c\cdot\frac{2}{3}(M_f-m_f)$$
$$M_{f-g_0-g_1}-m_{f-g_0-g_1}\leq c\cdot(M_{f-g_0}-m_{f-g_0})\leq c^2\cdot(M_f-m_f).$$
Inductively we can construct a sequence $\{g_n|n=0,1,2,\cdots\}$ in $\mathcal{N}_3$ satisfying
$$||g_n||_\infty \leq c^n\cdot\frac{2}{3}(M_f-m_f),\ \ \ \ \ \ \ M_{f-\Sigma_{i=0}^n g_i}-m_{f-\Sigma_{i=0}^n g_i}\leq c^{n+1}\cdot(M_f-m_f).$$

Since a positive constant $c$ is strictly less than 1, this estimate implies that $\Sigma_{i=0}^\infty g_i$ converges uniformly on $K$ by the Weierstrass M-test and
$$\lim_{n\rightarrow \infty}(M_{f-\Sigma_{i=0}^n g_i}-m_{f-\Sigma_{i=0}^n g_i})=0.$$ Therefore the limit of $f-\Sigma_{i=0}^n g_i$ converging uniformly on $K$ must be a constant function, say $C$, and hence $$f=C+\Sigma_{i=0}^\infty g_i$$ can be approximated arbitrarily closely by $C+\Sigma_{i=0}^n g_i\in \mathcal{N}_{3}$ for sufficiently large $n$. This completes the proof.
\end{proof}

\section{Proof for finite $K$}

When $K$ is finite, point separation lemmas are enough to give a constructive proof of the sharp result $\overline{\mathcal{N}_2}=C(K)$.

\begin{Theorem}
Let $K=\{\bold{p}_1,\cdots,\bold{p}_N\}$ be a finite subset of $\Bbb R^n$ and $\mathcal{N}_{k}$ for each $k$ under sigmoidal or ReLU activation be the previously defined subspace of $C(K)$ with the sup norm.
For any $T\in C(K)$ and any $\epsilon>0$, there exists $F\in \mathcal{N}_{2}$ such that $||F-T||_\infty<\epsilon$.
\end{Theorem}
\begin{proof}
We give a proof when the activation function is $\sigma$. The proof for $\tau$ case can be obtained from the following just by replacing $\sigma$ and Lemma \ref{pt-closedset-1} with $\tau$ and Lemma \ref{pt-closedset-2} respectively.

Let $t_1<t_2<\cdots<t_m$ for $m\leq N$ be all the elements of $T(K)$ and set $d:=t_m-t_1$. If $m=1$ and $d=0$, then take $F\equiv t_1$. Otherwise assume $d>0$ (and hence $N\geq 2$). WLOG we may also assume that $\epsilon< \frac{dN}{2}$.

For each $i=1,\cdots,N$, by applying Lemma \ref{pt-closedset-1} with $A=\{\bold{p}_i\}$ and $B=K-\{\bold{p}_i\}$ one can construct $h_i\in \mathcal{N}_{2}$ such that
$$0\leq h_i|_{K-\{\bold{p}_i\}}<\frac{\epsilon}{dN}<1-\frac{\epsilon}{dN}<h_i(\bold{p}_i)\leq 1.$$
Define $$F:=\sum_{i=1}^N(T(\bold{p}_i)-t_1)h_i+t_1\in \mathcal{N}_{2}.$$
Then for each $i$ $$(T(\bold{p}_i)-t_1)(1-\frac{\epsilon}{dN})+t_1\leq F(\bold{p}_i)\leq (T(\bold{p}_i)-t_1)+\sum_{j\ne i}(T(\bold{p}_j)-t_1)\frac{\epsilon}{dN}+t_1$$ where the first terms of both sides come from the contribution from the $i$-th term in the summation.  Using $T(\bold{p}_i)-t_1\leq d$ for any $i=1,\cdots,N$,
$$(T(\bold{p}_i)-t_1)-d\frac{\epsilon}{dN}+t_1\leq F(\bold{p}_i)\leq T(\bold{p}_i)+(N-1)d\frac{\epsilon}{dN}$$
and hence $$T(\bold{p}_i)-\epsilon< F(\bold{p}_i)< T(\bold{p}_i)+\epsilon,$$
completing the proof.
\end{proof}

\bigskip

\noindent{\bf Declarations}

\medskip

\noindent{\bf Data Availability} : Data sharing is not applicable to this article as no new data were created or analyzed in this study.

\noindent{\bf Conflicts of Interest} : The author has no relevant financial or non-financial interests to disclose.

\noindent{\bf Funding} : No funding has been provided for this study.

\vspace{0.5cm}


\bigskip



\begin{thebibliography}{aa}

\bibitem{barron} A. R. Barron, {\it Universal Approximation Bounds for Superpositions of a Sigmoidal Function}, IEEE Trans. Information Theory {\bf 39} (1993), NO.3, 930-945.

\bibitem{Besse} A. Besse, {\it Einstein manifolds}, Springer-Verlag, New York, 1987.


\bibitem{cybenko} G. Cybenko, {\it Approximation by superpositions of a sigmoidal function}, Math. Control Signal Systems {\bf 2} (1989), 303-314.


\bibitem{hornik} K. Hornik, M. Stinchcombe, and H. White, {\it Multilayer feedforward networks are universal approximators}, Neural Networks {\bf 2} (1989), Issue 5, 359–366.

\bibitem{ito} Y. Ito, {\it Approximation of functions on a compact set by finite sums of a sigmoid function without scaling}, Neural Networks {\bf 4} (1991), Issue 6, 817-826.

\bibitem{kim} N. Kim, C. Min, and S. Park, {\it Minimum width for universal approximation using ReLU networks on compact domain}, International Conference on Learning Representations (ICLR), 2024.

\bibitem{leshno}  M. Leshno, V. Y. Lin, A. Pinkus, and S. Schocken, {\it Multilayer feedforward networks with a nonpolynomial activation function can approximate any function}, Neural Networks {\bf 6} (1993), Issue 6, 861-867.

\bibitem{li} X. Li, {\it Simultaneous approximations of multivariate functions and their derivatives by neural networks with one hidden layer}, Neurocomputing {\bf 12} (1996), Issue 4, 327-343.

\bibitem{monico}  C. Monico, {\it An elementary proof of a universal approximation theorem}, arXiv:2406.10002 [cs.LG].

\bibitem{munk} J. R. Munkres, {\it Topology}, Pearson Education International, 2000.

\bibitem{hidalgo}  E. Paluzo-Hidalgo, R. Gonzalez-Diaz, and M. A. Gutiérrez-Naranjo {\it Two-hidden-layer feed-forward networks are universal approximators: A constructive approach},  Neural Networks {\bf 131} (2020), Issue 5, 29–36.

\bibitem{pinkus} A. Pinkus, {\it Approximation theory of the mlp model in neural networks}, Acta Numerica {\bf 8} (1999), 143-195.

\bibitem{sonoda} S. Sonoda and N. Murata, {\it Neural network with unbounded activation functions is universal approximator}, Appl. Comput. Harm. Anal. {\bf 43} (2017), Issue 2, 233-268.

\bibitem{yarotsky}  D. Yarotsky, {\it Error bounds for approximations with deep ReLU networks}, Neural Networks {\bf 94} (2017), 103-114.

\bibitem{yun} C. Yun, S. Sra, and A. Jadbabaie, {\it Small ReLU networks are powerful memorizers: a tight analysis of memorization capacity}, NeurIPS 2019.

\end{thebibliography}
\end{document}